%% file: main.tex
\documentclass{article}
\usepackage{preprint,times}
\input{math_commands.tex}

\usepackage{url}
\usepackage{graphicx}
\usepackage{tikz}
\usetikzlibrary{arrows.meta}
\definecolor{ovharn}{HTML}{0072B2}\definecolor{ovwts}{HTML}{E69F00}\definecolor{ovstep}{HTML}{009E73}
\AtBeginDocument{\setlength\abovedisplayskip{5pt plus 1pt minus 2pt}\setlength\belowdisplayskip{5pt plus 1pt minus 2pt}\setlength\abovedisplayshortskip{1pt}\setlength\belowdisplayshortskip{3pt}}
\usepackage{booktabs}
\usepackage{amsmath}
\usepackage{amsthm}
\newtheorem{proposition}{Proposition}
\usepackage{amssymb}
\usepackage{float}
\usepackage{xcolor}
\usepackage{multirow}
\usepackage{wrapfig}

\title{Harness Evolution Hits a Ceiling: When Weight Training Should Begin}

\preprintfinalcopy
\author{%
  Yuan Tian\thanks{Equal contribution.}\hspace{0.55em}\thanks{Corresponding authors: \texttt{yuan.tian.research@gmail.com}, \texttt{binghuresearch@gmail.com}.} \\
  Independent Researcher \\
  \And
  Bing Hu\footnotemark[1]\hspace{0.55em}\footnotemark[2] \\
  Independent Researcher \\
  \And
  Hao Wang \\
  University of California, Berkeley \\
  \AND
  Binghang Lu \\
  Purdue University \\
  \And
  Fang Wu \\
  Stanford University \\
}

\begin{document}
\maketitle
\lhead{Preprint. Under review.}
\suppressfloats[t]

\begin{abstract}
Improving a long-horizon LLM agent means evolving the harness around a frozen model or training its weights. We follow one line through both: let a self-evolving harness make the system stronger first, then cross seed and evolved harnesses with base and trained weights to learn which gains the trained model keeps and which still need the runtime. We show that the right lever can be read off the agent's failure composition. Labelling failed trajectories by the first signal that fires separates process failures (blocked calls, loops, exhausted step budgets) from content failures (a delivered plan that is poor). Harness evolution repairs the former, the behaviour it instils can be trained into the weights, and content failures are what weight training is for. On DeepPlanning, a verifiable planning benchmark with a held-out split, a self-evolving harness loop, in which an LLM proposes edits to its own harness and a judge verifies them, lifts the held-out score of Qwen3.5-4B from 0.16 to 0.30 and of Qwen3.5-9B from 0.32 to 0.44, and the gain lands where the claim places it: for 4B, held-out delivery rises from 55\% to 90\% and loops fall, while content failures are left for the weights. Ablations pin the gain on capability-granting components rather than on enforcement, and the loop's own evidence package shows why it stops where it does. LoRA adapters trained on evolved-harness trajectories then internalise the gain: under the original harness they add +0.13 on held-out tasks for both sizes; on 4B they stack with the harness to more than double the held-out score, and on 9B the adapter alone matches the full evolution line, cutting content failures from a quarter of trajectories to one in twenty. A placebo adapter trained on answer-shuffled trajectories falls below the base model, so the gain is in the trajectory content; full fine-tuning on a larger trajectory set adds a further +0.04. The same loop transfers to WebArena-Lite (+0.09 on 117 unseen tasks), where its gain lives in what the model sees, two pages of history instead of one, and adapters trained on those trajectories do not add to it. The result is a diagnose-then-intervene rule applied twice: read the failure composition, and evolve the harness for process failures and train the weights for content failures; then read what the accepted edits changed, and train in the gains that changed what the model writes while keeping the harness for the gains that changed what it sees. Scores are four-rollout means against fresh anchors and comparisons are drawn within one night, except where marked, across eight models from six families and two benchmarks.
\end{abstract}

\section{Introduction}
\label{sec:intro}

A long-horizon LLM agent can be improved in two places. One can change the \emph{harness}, the code and text around a frozen model: system prompts, memory tools, checkpoints, output gates, the tools it may call. Or one can change the \emph{weights}, by fine-tuning on trajectories. Each lever has its own literature, self-evolving harnesses \citep{ahe,adas,rrsi,harnesscompass} and trajectory fine-tuning \citep{agenttuning,fireact,webrl}, and concurrent work has begun to alternate the two \citep{whale,coharness,harness_zero}. A self-evolving harness can raise the system's score without touching the weights, but that score does not say whether the model still needs the runtime to reach it. Our line is therefore: let self-evolution make the harnessed system stronger first, then cross seed and evolved harnesses with base and trained weights to learn which gains the trained model keeps and which still need the runtime. Both levers are expensive to evaluate and the per-rollout score is noisy, so the rule for \emph{which} lever to pull must be grounded in measurements that survive the noise.

This paper builds that rule empirically on DeepPlanning \citep{deepplanning}, a travel-planning benchmark with deterministic programmatic scoring, using 48 development tasks for every decision and 72 held-out tasks no loop ever saw. We first fix the measurement protocol: with a per-rollout standard deviation of 0.023, every configuration is a four-rollout mean, differences below 0.035 are noise, and a fresh anchor follows every accepted change. We then label every trajectory by the first failure signal that fires (Section~\ref{sec:taxonomy}): unbacked assertions in the plan (F1), the harness's own interventions getting in the way (F2), loops and exhausted step budgets (F3), or a delivered plan that is simply poor (F4). The hypothesis the paper tests is that the harness lever removes the process classes (F2, F3) and leaves the content class (F4) in place, that the behaviour the harness instils can be trained into the weights, and that the content class is the target for weight training (Figure~\ref{fig:overview}).

\begin{figure}[t]
\centering
\resizebox{0.90\linewidth}{!}{%
\begin{tikzpicture}[font=\scriptsize, >={Latex[length=1.6mm]}, thick,
  box/.style={draw=ovstep!80!black, rounded corners=2pt, align=center, inner sep=4pt, text width=#1, fill=ovstep!12},
  lab/.style={font=\tiny, inner sep=1pt, fill=white}]
\node[box=2.1cm] (bench) at (0,0) {\textbf{Benchmark}\\tasks with a programmatic scorer};
\node[box=2.7cm] (diag) at (3.4,0) {\textbf{Diagnose}\\label every trajectory F1--F4;\\delivery $\times$ quality};
\node[box=4.3cm, draw=ovharn, fill=ovharn!10] (harn) at (8.2,0.85) {\textbf{\textcolor{ovharn!85!black}{Self-evolving harness}}\\an LLM proposes edits to its own harness; a judge verifies them};
\node[box=4.3cm, draw=ovwts, fill=ovwts!14] (wts) at (8.2,-0.85) {\textbf{\textcolor{ovwts!70!black}{Train the weights}}\\LoRA or full fine-tuning on trajectories from the evolved harness};
\node[box=2.6cm] (eval) at (13.0,0) {\textbf{Evaluate}\\seed vs.\ evolved harness $\times$ base vs.\ adapter, held-out};
\draw[->] (bench) -- (diag);
\draw[->] (diag.east) -- (harn.west) node[lab, midway, above, sloped, text=ovharn!85!black] {F2, F3};
\draw[->] (diag.east) -- (wts.west) node[lab, midway, below, sloped, text=ovwts!70!black] {F4};
\draw[->] (harn.south) -- (wts.north) node[lab, midway, right] {trajectories};
\draw[->] (harn.east) -- (eval.west);
\draw[->] (wts.east) -- (eval.west);
\end{tikzpicture}}
\caption{\textbf{Overview.} Trajectories under the seed harness are labelled by their first failure signal (Section~\ref{sec:taxonomy}) and the composition decides the lever: a self-evolving harness loop removes the process failures, the trajectories it produces train the weights for the content failures, and both levers are measured under the seed and the evolved harness on held-out tasks.}
\label{fig:overview}
\end{figure}
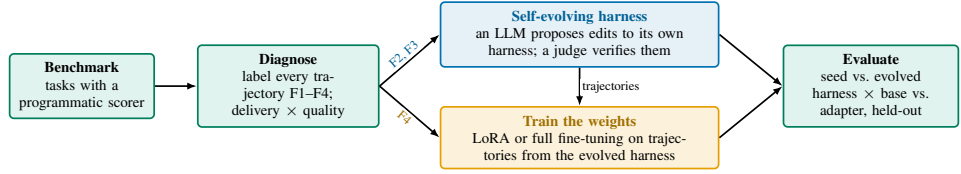

\textbf{Harness self-evolution} (a loop in the style of AHE~\citep{ahe}, Section~\ref{sec:setup}; Figure~\ref{fig:evolution}) lifts Qwen3.5-4B by +0.143 / +0.133 (dev / held-out) over the seed harness and Qwen3.5-9B by +0.056 / +0.119, and the gains land where the taxonomy places them (Figure~\ref{fig:frate_breakdown}): for 4B, delivery rises from 59\% to 88\% and loop failures (F3) fall from 64\% to 37\% of trajectories; for 9B, delivery was already 97\% and loops fall from 39\% to 17\%, so the gain shows up as better delivered plans. What the loop leaves behind is the content class: F4 rises from 10\% to 28\% of 4B trajectories and stays at 25\% for 9B.

\textbf{Why the loop stops at F4.} Component ablations (Section~\ref{sec:analysis}) show that the compliance-enforcing components (exit gate, checkpoint, hard blocks) are mostly within noise, whereas the capability-granting components (a state ledger, a distilled lessons file, invented tools) carry the gain. Yet the evidence package the proposal model reads leads with gate and checkpoint metrics and never aggregates the scorer's failure composition: the loop optimises what it can see, and what it can see does not include the cause of F4.

\textbf{Weight training} (LoRA on trajectories collected under the evolved harness, Section~\ref{sec:results}) transfers to held-out tasks: on the seed harness the adapter alone adds +0.13 for both sizes, matching or exceeding what harness evolution delivered; stacked on the evolved harness it adds a further +0.04 / +0.08 (dev / held-out) for 4B, while for 9B the adapter alone reaches the evolved-harness level, so at that size the two levers substitute for each other. Each adapter moves its own failure class (the 4B adapter raises delivery and cuts loops, the 9B adapter cuts F4 from 23--28\% of trajectories to 5--7\%); a placebo adapter trained on answer-shuffled trajectories falls 0.16 / 0.23 below the base model, and full fine-tuning adds about +0.04 on held-out over LoRA.

\textbf{Contributions.} (1) A measurement protocol for harness and weight interventions on a noisy long-horizon benchmark, with a bound on selection optimism. (2) Harness self-evolution on two model sizes with held-out validation, decomposed into delivery, loop reduction and plan quality; F4 is the class the loop left in place in both Qwen lines. (3) Component ablations on both sizes: capability-granting components carry the gain, enforcement mostly does not. (4) An account of why the loop stops at F4, from what its evidence package shows the proposer. (5) LoRA and full-parameter training on evolved-harness trajectories, with a placebo control and a rank ablation: the gain transfers to held-out tasks, stacks with the harness on 4B and substitutes for it on 9B, and each adapter moves the failure class the rule assigns to it; on WebArena-Lite, where the gain lives in what the model sees, it did not train in. (6) Generality: across nine evolution lines on eight models from six families, the seed measurements separate the five lines where the harness lever paid (+0.05 to +0.14) from the three where it paid nothing; the loop transfers to a second benchmark.

\section{Background and Related Work}
\label{sec:background}

\subsection{Failure Analysis of LLM Agents}
Taxonomies of agent failure have grown fine-grained: MAST \citep{mast} annotates fourteen failure modes of multi-agent systems, TRAIL \citep{trail} and AgentErrorTaxonomy \citep{agentdebug} label reasoning, execution and planning errors inside traces, and \citet{hidden_failures} catalogue 78 failure types. These are organised by the \emph{kind} of error and need an LLM or a human annotator; ours is organised by the \emph{lever that removes} the error and is computed by a deterministic checker from process signals and the scorer. Ecdysis \citep{ecdysis} comes closest, using recurrence across tasks to decide which failures the harness should absorb; we test the split by applying each lever and reading which class it moves.

\subsection{Self-Evolving Harnesses and Automated Agent Design}
A fast-growing line treats the harness around a frozen model as the optimisation target: ADAS \citep{adas} and AFlow \citep{aflow} search agent programs and workflow graphs, and AHE \citep{ahe} introduced the propose-verify loop we build on. The loop has since been regularised against overfitting \citep{rrsi,harnesscompass}, and challenged to beat matched-budget test-time search \citep{rethinking_harness_eval}; memory architectures \citep{memgpt,auto_mem} and runtime governance \citep{state_aware_runtime} are particular harness components, and \citet{sham_control} price planning guidance against a shuffled-text control. We add which class of failure the loop removed, which it left, and what its proposer could see.

\subsection{Training Agents on Trajectories, and Moving Harness Gains into the Weights}
Fine-tuning agents on trajectories has a settled recipe: AgentTuning \citep{agenttuning} and FireAct \citep{fireact} distil GPT-4 trajectories into small models, WebRL \citep{webrl} trains web agents on WebArena-Lite with an online curriculum, and STaR-style self-improvement \citep{zelikman2022star} trains on the model's own filtered outputs, the recipe we use; \citet{distill_system2} moved a prompting scaffold's gain into the weights for single-turn tasks. Concurrent work brings the question to agent harnesses: Harness-Zero \citep{harness_zero} distils an optimised harness into the weights and removes it at deployment, Co-Harness \citep{coharness} and WHALE \citep{whale} alternate harness search and weight updates, \citet{harness_posttraining} show that harness-aware post-training is more robust to tool shift, and \citet{coevolve_harness} find that imitating an expert under an evolved harness can regress. These works switch lever by schedule or by training signal; we switch on the failure composition, obtain the removable-harness result with plain fine-tuning on the model's own trajectories, show that it is size-dependent, and add a placebo adapter trained on answer-shuffled trajectories. LoRA \citep{hu2022lora} and QLoRA \citep{dettmers2023qlora} supply the machinery; long-horizon RL credit assignment \citep{schulman2017ppo,shao2024grpo,himp,hcapo} is the route we do not take.

\subsection{Evaluating Agents Under Noise}
DeepPlanning \citep{deepplanning} scores plans programmatically and admits a clean development / held-out split; WebArena \citep{webarena}, AgentBench \citep{agentbench} and AgentBoard \citep{agentboard} are broader but rely on LLM judges or do not expose failure composition. \citet{agents_that_matter} document inadequate hold-outs, \citet{error_bars} ask for standard errors and paired comparisons, \citet{winners_curse} formalise the winner's curse of adaptive search, and \citet{rethinking_harness_eval} show that harness evolution overfits when search and report share a benchmark.

\textbf{What this paper adds.} Relative to these lines, this paper (i) labels failures by the lever that removes them and shows, by applying both levers, where the split holds; (ii) measures both levers on the same models and tasks under one calibrated noise model; (iii) traces the harness ceiling to what the loop's evidence package shows the proposer; and (iv) turns the results into a rule that reads the lever off the failure composition.

\section{A Diagnostic Framework}
\label{sec:framework}
\subsection{The failure taxonomy}
\label{sec:taxonomy}
Each trajectory is labelled post hoc by an automated checker: a trajectory whose delivered plan scores at least 0.5 is \emph{healthy}, and every other trajectory receives the first of the following signals that fires:
\begin{itemize}\setlength\itemsep{0.05em}
  \item \textbf{F1, evidence}: at least half of the plan's factual claims are unbacked by anything retrieved (unbacked-citation ratio $\geq 0.5$), or the agent passed the exit gate with unresolved items.
  \item \textbf{F2, harness friction}: a harness intervention itself blocked the agent (a forced checkpoint unlock, or ten or more blocked tool calls) and the score is low.
  \item \textbf{F3, execution control}: the same tool call repeated three or more times, or the trajectory used 90\% or more of its step budget, or the harness terminated it.
  \item \textbf{F4, planning capacity}: none of the signals above fired; in practice a plan was delivered and scores below 0.5 (rarely, an undelivered trajectory with no process signal).
\end{itemize}
Separately, we record the \emph{delivery rate}, the fraction of trajectories that returned a parsable plan at all, and the \emph{quality among delivered plans}, the mean score over those trajectories; the composite score is their product. This is the hypothesis the paper tests: F2 and F3 are the classes a harness can act on directly, and F4 is the class it did not reach in either Qwen line, since by construction nothing in the trajectory's process signals distinguishes an F4 trajectory from a healthy one; a harness gain that consists of removed F2 and F3 trajectories is a behaviour, which the weights can learn. Table~\ref{tab:f-rates} (Appendix~\ref{app:frates}) gives the composition before and after evolution for both Qwen sizes.

\subsection{Two identities behind the numbers}
\label{sec:identities}
\textbf{Selection optimism.} A round keeps the best of several candidate readings. A reading is the mean of $m$ rollouts of the 48-task score, with single-rollout standard deviation $\sigma = 0.023$ (Appendix~\ref{app:noise}), so it carries noise of scale $\sigma/\sqrt{m}$.
\begin{proposition}[selection optimism]
\label{prop:optimism}
Let $X_k = \mu_k + \varepsilon_k$, $k = 1, \dots, K$, be the readings of $K$ candidates with true means $\mu_k$ and zero-mean noises $\varepsilon_k$, each sub-Gaussian with parameter $s$. If the candidate with the largest reading is kept, $k^\star = \arg\max_k X_k$, then
\begin{equation}
\mathbb{E}\big[X_{k^\star} - \mu_{k^\star}\big] \;\le\; s\sqrt{2\ln K}.
\label{eq:optimism}
\end{equation}
\end{proposition}
\begin{proof}
$X_{k^\star} - \mu_{k^\star} = \varepsilon_{k^\star} \le \max_k \varepsilon_k$, and the expected maximum of $K$ zero-mean $s$-sub-Gaussian variables is at most $s\sqrt{2\ln K}$ \citep{boucheron2013concentration}.
\end{proof}
With 10 to about 30 candidates per line and one rollout per reading ($s = \sigma$) the bound is 0.05--0.06, above the 0.035 line, which is why the three single-rollout lines of Section~\ref{sec:other-models} were re-measured under the full protocol before being reported; with four rollouts ($s = \sigma/2$) it is 0.025--0.030, inside the line, and the fresh anchor after every accept removes the optimism from the next round's anchor (Figure~\ref{fig:winners-curse}, Appendix~\ref{app:pitfalls}, shows the gap in our two lines). The bound concerns measurement noise only; whether a gain holds on other tasks is what the 72 held-out tasks answer.

\textbf{Score identity.} Let $r$ be the delivery rate and $\bar q$ the mean score among delivered plans. On DeepPlanning an undelivered plan scores zero, so the composite score is $S = r\,\bar q$ exactly, and a change from $(r, \bar q)$ to $(r', \bar q')$ decomposes as
\begin{equation}
\Delta S \;=\; \bar q\,\Delta r \;+\; r'\,\Delta \bar q .
\label{eq:identity}
\end{equation}
A lever that only moves $r$ is capped at $\bar q$, the quality of what the model already writes; a lever that raises $\bar q$ is worth $r'\,\Delta\bar q$, more when delivery is already high. The two factors are complementary, and Table~\ref{tab:decomp} (Appendix~\ref{app:frates}) reports both for the Qwen lines and MiniCPM.

\section{Experimental Setup}
\label{sec:setup}

\subsection{Benchmark, models and measurement}
DeepPlanning \citep{deepplanning} poses multi-day travel-planning tasks that require tool calls (flights, hotels, attractions, distances), state tracking across many steps, and a final structured plan that is scored programmatically against hard constraints. We use 120 tasks split once, before any experiment, into 48 development tasks and 72 held-out tasks. All evolution, training and ablation decisions use only the 48 development tasks; the 72 held-out tasks are evaluated only for reporting. The official benchmark loop is replaced by our own harness that exposes state to the agent and logs every step; the official tools and scorer are unchanged. Our primary models are Qwen3.5-4B and Qwen3.5-9B, served locally with thinking disabled at temperature 0.7. Three further families (MiniCPM5-2B, Spark-X2.5-4B, Ministral-3-8B) and Gemma~4~12B test the other branch of the rule (Appendix~\ref{app:boundary}).

Every configuration is scored as the mean of four rollouts. Across 26 same-night cells the per-rollout standard deviation has median 0.023 (worst 0.042; Appendix~\ref{app:noise}), so we treat a difference between two 4-rollout means as resolvable only above 0.035 and write ``within noise'' below it; comparisons are made between cells run on the same machine on the same night, since cross-night drift of one configuration reaches 0.05; the exceptions are marked where they occur and listed under Limitations.

\subsection{Harness self-evolution protocol}
\label{sec:protocol}
We run a self-evolving loop in the style of AHE \citep{ahe}: an LLM reads the evidence from the current harness's own runs and rewrites the harness. One round: (1) run the current harness on the 48 development tasks (4 rollouts); (2) an evidence package is built from the run (score, healthy rate, F1--F4 counts, process metrics, worst cases) and handed to a proposal model (Kimi K3), which returns two candidate edits, each with a manifest and its own rollback condition; (3) each candidate is evaluated with 4 rollouts and judged against the round's anchor by fixed house rules (reject if the mean drops by more than 0.02, if the healthy rate drops by 0.05, or if the no-submission or repeat-call rate rises by 0.05) plus the proposer's own rollback condition; (4) the better accepted candidate is retained, and a \emph{fresh} 4-rollout anchor of the new harness is measured before the next round. The editable surface is the harness definition only: prompt files, policy parameters, middleware and tools. Each line starts from the same seed harness (a state ledger, an exit gate, a mid-trajectory checkpoint and four hard blocks). The 4B line ran 13 rounds (26 candidates) and the 9B line 5 rounds (10 candidates); full logs are in Appendices~\ref{app:log4b} and~\ref{app:log9b}.

\subsection{Weight training}
\label{sec:training}
Training data are trajectories collected under the \emph{evolved} harness on the 48 development tasks: 8 rollouts per task for 4B and 12 for 9B, keeping the three best trajectories per task above a score threshold (0.5 for 4B, 0.6 for 9B), which gives 82 trajectories over 33 tasks for 4B and 98 over 40 tasks for 9B. Each assistant step is one sample; the loss is on the assistant output only. LoRA uses rank 32, $\alpha = 64$, dropout 0.05 and learning rate $10^{-4}$, on the attention projections (qkvo) or attention plus MLP (qkvo+MLP), taken at the end of the first epoch; full-parameter fine-tuning uses $2\times10^{-6}$ with warm-up and cosine decay (Appendix~\ref{app:training}). The adapters in Table~\ref{tab:lora} are evaluated on both harnesses and both splits on the same night as their base model; those of the second textbook (Appendix~\ref{app:training}) on the seed harness only.

\section{Results}
\label{sec:results}

\subsection{Harness self-evolution works and transfers}
\label{sec:evolution}
\begin{figure}[t]
\begin{center}
\includegraphics[width=0.98\linewidth]{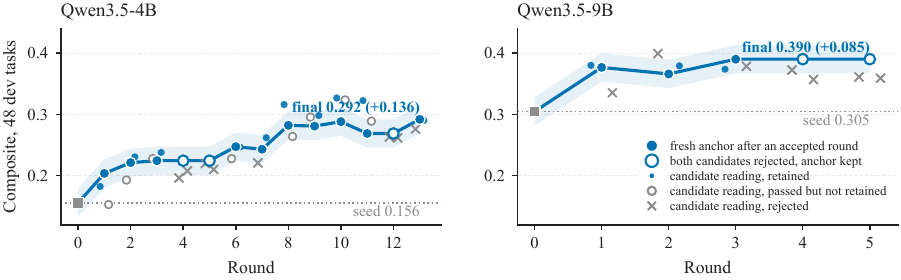}
\end{center}
\caption{\textbf{Per-round fresh anchors of the two evolution lines} (48 development tasks, 4-rollout means; band $= \pm 0.023$; filled: edit accepted, fresh anchor measured; hollow: both candidates rejected, anchor kept). Small markers are the two candidates of each round at their own readings (dot retained, hollow circle passed but not retained, cross rejected); the gap between a retained reading and the anchor that follows is the selection optimism of Proposition~\ref{prop:optimism}. The 4B dip at round 11 is a re-anchor after the context window was raised to 262k tokens.}
\label{fig:evolution}
\end{figure}
Figure~\ref{fig:evolution} shows the per-round anchors of both lines and the readings of both candidates in each round. The 4B line accepted an edit in 10 of 13 rounds (18 of 26 candidates passed the judge) and moved from 0.156 to 0.292 (+0.136); the 9B line accepted 3 of 10 candidates over 5 rounds and moved from 0.305 to 0.390 (+0.085). Table~\ref{tab:evolution-summary} (Appendix~\ref{app:frates}) re-measures seed and final harness outside the loop on both splits: the evolved harness transfers to held-out tasks (4B held-out gain +0.133 against dev +0.143; 9B +0.119 against +0.056).

Where does the gain come from? Table~\ref{tab:decomp} (Appendix~\ref{app:frates}) splits each score into delivery rate and quality among delivered plans, and Table~\ref{tab:f-rates} gives the failure composition. For 4B, delivery rises from 59\% to 88\% and loop failures (F3) fall from 64\% to 37\% of trajectories; Equation~\ref{eq:identity} assigns 0.077 of the +0.136 to delivery and 0.058 to the higher mean among delivered plans (0.266 to 0.332), and a paired comparison between the bare model and the full evolved harness on the tasks both deliver (Appendix~\ref{app:ablation}) moves by only +0.011 on dev and +0.022 on held-out, consistent with the higher mean reflecting which plans are now delivered rather than better plans on the same tasks. For 9B, delivery was already 97\% and the gain is a better plan among delivered ones (0.309 to 0.392), the $r'\,\Delta\bar q$ term, with loops falling from 39\% to 17\%. In neither line did the loop reduce F4: it grows from 10\% to 28\% of 4B trajectories as loops are converted and stays at 25\% for 9B, where it is the largest failure class at the end; for 4B the residual F3 (37\%) is still larger (Figure~\ref{fig:frate_breakdown}).
\begin{figure}[t]
\begin{center}
\includegraphics[width=0.92\linewidth]{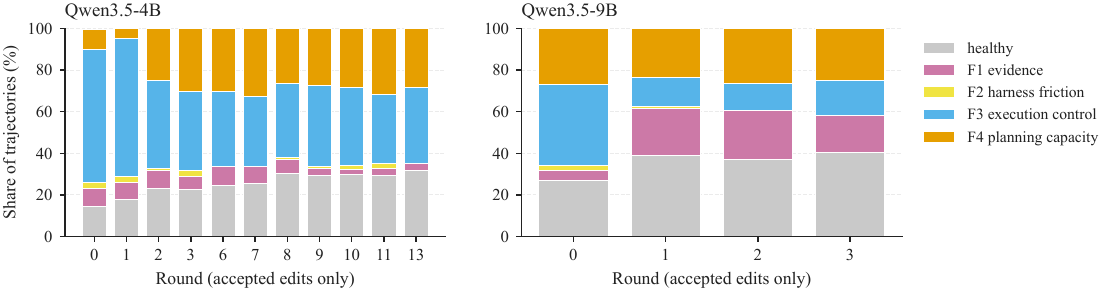}
\end{center}
\caption{\textbf{Failure composition of the fresh anchor after each accepted round}, from the same per-round history files as Figure~\ref{fig:evolution}. Loops (F3) fall and the healthy share rises in both lines; F4 is never reduced. In the 9B line the lessons file raises the share of plans with unbacked claims (F1).}
\label{fig:frate_breakdown}
\end{figure}

\subsection{Training internalises the harness gain}
\label{sec:lora}
Table~\ref{tab:lora} (plotted as Figure~\ref{fig:lora} in Appendix~\ref{app:training}) reports the same-night evaluation of the LoRA adapters trained on evolved-harness trajectories (Section~\ref{sec:training}). Four findings. \textbf{(i) The adapter transfers.} On the seed harness, the better adapter lifts held-out scores by +0.126 (4B, qkvo) and +0.131 (9B, qkvo+MLP), and development scores by +0.151 and +0.050; seven of the eight adapter-on-seed-harness cells are positive. \textbf{(ii) Stacking on 4B, substitution on 9B.} On the evolved harness, the 4B adapter still adds +0.041 (dev) and +0.082 (held-out); the 9B adapter adds +0.010 and +0.007, within noise, as is the other 9B adapter. For 9B either lever alone reaches 0.44--0.45 on held-out. \textbf{(iii) Content, not adapter.} The placebo (Table~\ref{tab:placebo}, Appendix~\ref{app:training}), trained with the identical recipe on the same trajectories with the assistant outputs permuted across samples, scores 0.158 below the base model on dev and 0.227 below on held-out, and its trajectories run 52--65 rounds instead of 19. Ranks 16, 32 and 64 all gain on held-out (Appendix~\ref{app:training}); full-parameter fine-tuning on a larger second textbook beats LoRA by +0.044 (4B) and +0.042 (9B) on held-out. \textbf{(iv) Each adapter moves its own failure class} (Table~\ref{tab:f-rates-lora}, Appendix~\ref{app:training}). The 4B adapter reproduces the harness's behavioural effect: delivery rises from 52\% to 78\% on dev and from 55\% to 78\% on held-out, loops fall by 18--19 points, and F4 is unchanged. The 9B adapter, on a model whose delivery is already saturated, moves the content class instead: F4 falls from 23\% to 7\% on dev and from 28\% to 5\% on held-out, and the freed trajectories become healthy (34\% to 46\%, 28\% to 53\%): training reaches the class the loop did not.

\textbf{Which placement wins follows what the textbook has to teach.} The attention-only adapter wins on 4B and the adapter that also trains the MLP blocks wins on 9B (Table~\ref{tab:lora}); we select the placement per model on the development split. The 4B textbook was collected by making a model that delivers half of its plans deliver, so what it mostly carries is that behaviour, and the attention-only adapter learned it better (delivery 78\% against 67\%, F3 40\% against 51\%; Table~\ref{tab:f-rates-lora}). The 9B textbook comes from a model that already delivers, so it carries plan content, and both adapters cut F4 (from 23\% and 28\% on dev and held-out), the one that also trains the MLP blocks, where language models hold most of their factual associations \citep{geva2021kv}, by more (to 7\% and 5\% against 13.5\% and 16\% for attention only). A richer textbook should then make the MLP placement win on 4B as well, which the second textbook shows (Table~\ref{tab:training}).

\input{tab_lora_results}

\subsection{The same loop on a second benchmark: WebArena-Lite}
\label{sec:webarena}
WebArena-Lite \citep{webarena,webrl} poses 165 web-navigation tasks on five self-hosted sites, with the official environments and scoring; natural-language answers are judged by an LLM. We split it once into 48 development and 117 test tasks and ran the same loop with Qwen3.5-9B under the protocol of Section~\ref{sec:protocol}, with a line measured on this benchmark ($\sigma = 0.045$, line 0.064). The seed harness reproduces the benchmark agent's interface and adds the ledger, checkpoint and exit gate of the DeepPlanning seed harness. Thirteen rounds accepted ten edits and moved the development anchor from 0.234 to 0.375; one interface edit in round 10, keeping the two most recent full pages in the context instead of one, accounts for +0.12 of it.

\textbf{The gain transfers, with the same shape.} Table~\ref{tab:webarena} (Appendix~\ref{app:boundary}) reports seed and closing harness on the 117 test tasks, measured on one night: 0.256 to 0.344 (+0.088), delivery from 49\% to 77\%, the healthy share from 23\% to 34\%, and accuracy among delivered answers unchanged within noise (0.47 to 0.44). The failures change kind: at the seed, 165 of 237 undelivered trajectories stopped on a repeated action or three unparseable actions, under the closing harness 13 of 108, the rest exhausting step or call budgets. The harness removed the process failures and carried the agent further into each task, and the correctness of what it submits did not move. Removing the three seed components from the closing harness costs 0.10 on the development tasks (0.396 to 0.297, same night).

\textbf{Where the gain lives decides whether it can go into the weights.} LoRA and full-parameter adapters were trained on the 157 full-score trajectories (35 tasks) that the closing harness produced on the development tasks. Measured the night after the base cells (the closing-harness base anchor re-measured alongside them: 0.375 against 0.396), they are within the 0.064 line of the base model on the 117 test tasks under both harnesses (0.344 against 0.327 and 0.331; 0.256 against 0.216 and 0.231); delivery rises further, to 81\% and 88\%, and accuracy among delivered answers falls, to 0.39 and 0.37. The gain on this benchmark lives in what the model sees, two full pages in the context instead of one, and an adapter that is not shown the second page did not reproduce what the model did with it. Set against Section~\ref{sec:lora}, where every accepted edit changed what the model does (tools, prompts, lessons, blockers) and the adapters carried the gain, this gives a reading of when a harness gain can be moved into the weights: gains from edits to what the model writes trained in, gains from edits to what it sees stayed in the harness, and the surface the accepted edits touched is known before any training is run.

\subsection{Where the harness lever pays, across nine lines}
\label{sec:other-models}
\begin{figure}[t]
\centering
\includegraphics[width=0.82\linewidth]{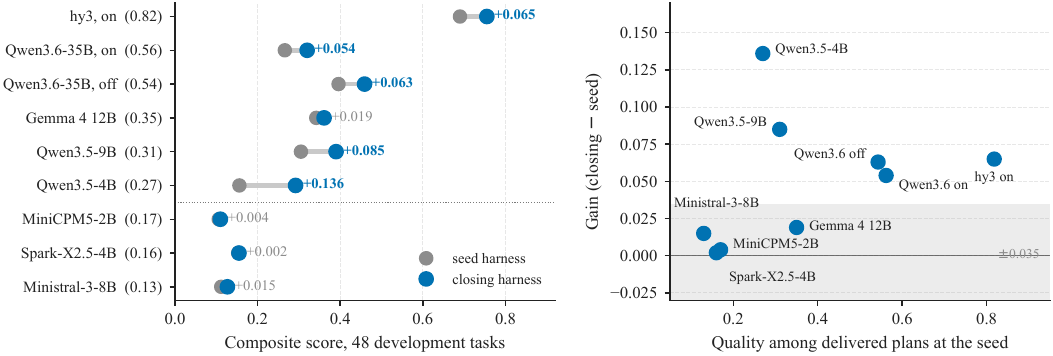}
\caption{\textbf{Harness evolution on the nine DeepPlanning lines reported} (48 development tasks). Left: seed (grey) and closing (blue) score of each line, ordered by the quality of the plans the seed harness delivers (in parentheses; on/off = thinking on or off). Every line is measured under the protocol of Section~\ref{sec:protocol}, the three early lines and Gemma by a same-night re-measurement of seed and closing compositions (Appendix~\ref{app:boundary}). Right: the gain against the seed plan quality; band $= \pm 0.035$, the resolution of the protocol on the Qwen3.5 cells; the Qwen3.6 thinking-on cells are noisier and their +0.053 is within that noise (Appendix~\ref{app:boundary}).}
\label{fig:cross-model}
\end{figure}
Figure~\ref{fig:cross-model} collects the nine DeepPlanning evolution lines reported in this paper, ordered by the quality of the plans the seed harness delivers. Three were evolved before the protocol of Section~\ref{sec:protocol}, with one rollout per reading: two on Qwen3.6-35B-A3B (thinking on and off) and one on Tencent hy3 (thinking on); we re-measured their seed and closing compositions on one night with four rollouts each, and the gains re-measure at +0.063 (Qwen3.6 off), +0.053 (Qwen3.6 on) and +0.065 (hy3), all of it delivery; the Qwen3.6 thinking-on cells are noisier than the Qwen3.5 cells (per-rollout standard deviation up to 0.07), so that gain is within its own noise (Appendix~\ref{app:boundary}). The seed measurements of Section~\ref{sec:taxonomy}, all available before a round is run, separate the nine lines into three groups, and the outcomes follow the grouping. Where the seed harness delivers plans scoring 0.27 or better and F4 is not yet the largest failure class, five lines from a 4B model to hy3, the harness lever paid +0.063 to +0.136 on the development tasks in four and +0.053, within its own noise, in the fifth. Where the delivered plans score 0.13--0.17, three lines, it paid nothing: MiniCPM5-2B, Spark-X2.5-4B and Ministral-3-8B closed within noise of their seeds after five rounds each, although the harness still moved delivery on MiniCPM5-2B (61\% to 71\%) and F1 on Spark-X2.5-4B (21\% to 7\%), with the freed trajectories landing in F4 (Table~\ref{tab:f-rates-other}): only a model that can write a plan worth scoring can cash the move. Where delivery is saturated and F4 is already the largest class among plans worth scoring, one line, Gemma~4~12B, the rule sends the next move to the weights; the harness bought +0.019, within noise, and what its trajectories then taught the weights is in Appendix~\ref{app:boundary}. Two frontier API models run separately by a co-author (Appendix~\ref{app:external}) point the same way: under an evolved harness Kimi K3 gains +0.058 / +0.049 (dev / held-out) with most of the gain in quality among delivered plans, the 9B pattern, and GLM~5.3 reaches 0.512 / 0.488 (no bare control).

\section{Analysis}
\label{sec:analysis}

\subsection{What the harness components buy}
\label{sec:components}
Appendix~\ref{app:ablation} (Tables~\ref{tab:ablation-9b} and~\ref{tab:ablation-4b}) reports component ablations on both sizes, each set run in one night (the 9B set on two identical machines, cross-machine offset below 0.03), with the line's closing harness as the full configuration. For 9B, removing the mid-trajectory checkpoint, the exit gate or the four hard blocks leaves the development score within noise of the full harness (+0.019, +0.017, +0.017); removing the evolved lessons file costs 0.049 on dev and 0.076 on held-out, and removing the ledger, together with the controls that depend on it, costs 0.088 and 0.145. The ledger and the lessons are not additive: with neither or with the ledger alone the score is within noise of the bare model; with both it is +0.076 / +0.092 above it. For 4B, removing the loop's own middleware or its three policy rewrites leaves the score within noise; removing the three invented tools costs 0.071 / 0.090 and removing the lessons file 0.067 / 0.120; removing the exit gate costs nothing on dev but 0.077 on held-out, where delivery drops from 88\% to 73\%. Tools and lessons account for the largest losses; enforcement is within noise except the 4B exit gate, which holds held-out delivery.

\subsection{Why the loop stops at F4}
\label{sec:evidence_package}
The evolution loop can only optimise what it is shown, and what it is shown is the evidence package. Four properties of it bound what the proposer can see. (a) The package leads with intervention process metrics, the components the ablations show do not move the score; the F1--F4 counts come after them, and the scorer's per-check failure composition was never aggregated into the package. (b) The proposal request is a single call with no tools. (c) The keys that define which metrics are observed live outside the four directories the loop may edit. (d) The accepted edits target what the package makes visible: the 4B tools and prompts raise delivery and cut loops, the 9B lessons cut loops, and no edit changed the rate at which the 9B agent submits several drafts of a plan (25.5\% before, 25.0\% after), a failure the scorer sees but the package never surfaced. The ceiling is a property of this loop's evidence package, not of harness evolution in general; a first attempt to surface the failure composition to the proposer (Gemma~4~12B, Appendix~\ref{app:boundary}) did not produce a resolvable gain.

\section{Discussion}
\label{sec:discussion}

\textbf{A diagnose-then-intervene rule.} Measure the noise floor first, then the delivery rate, the quality among delivered plans and the F1--F4 composition on a development set, and branch. (1) Delivered plans score below about 0.18: the harness lever paid nothing on the three such models, moving delivery and F1 while the freed trajectories landed in F4; we did not train them. (2) Delivery well below saturation or F3 dominant: evolve the harness; its gains transferred to held-out tasks in both Qwen lines. (3) Delivery saturated and F4 dominant: train the weights; on 9B the adapter alone was worth as much as the five-round evolution line on held-out, and on 4B the two levers stack. (4) F4 residual after both: the loop as run here cannot see it; expose the scorer's failure composition to the proposer (Section~\ref{sec:evidence_package}). (5) A harness gain in hand: read what the accepted edits changed; where they changed what the model does, the adapter carried the gain (Section~\ref{sec:lora}), and where the gain came from what the model sees, it stayed in the harness (Section~\ref{sec:webarena}).

\textbf{Generality.} Across the nine lines of Section~\ref{sec:other-models}, the seed measurements (delivery, quality, failure composition) separate the five lines on which the harness lever paid (+0.05 to +0.14, the smallest within its own noise) from the three on which it paid nothing; on WebArena-Lite the harness gain held on 117 unseen tasks (+0.088) and training on its trajectories did not add to it.

\textbf{Limitations.} Applying the taxonomy needs trajectory-level process signals and a task-level scorer. The 4B ablation cells were run with two rollouts except the exit-gate cell, and the placebo control exists for 9B only. We do not compare the loop with a matched-budget test-time search baseline \citep{rethinking_harness_eval}. The benchmark's plan-to-JSON conversion is an LLM call that inflates scores by about 0.025 on the sample we audited (Appendix~\ref{app:parser}); every cell uses the same converter. Two comparisons cross nights: the base-model cells under the evolved harness on held-out tasks in Tables~\ref{tab:lora} and~\ref{tab:evolution-summary} are closing anchors from two nights earlier (two rollouts for 4B, so its held-out stacking gain is read against a 0.05 line), and the WebArena-Lite adapters were measured the night after their base cells (Table~\ref{tab:webarena}); every other comparison is same-night.

\bibliography{references}
\bibliographystyle{preprint}

\appendix

\section{Noise Floor and Measurement Protocol}
\label{app:noise}
\input{tab_noise}
The per-rollout standard deviation was measured on 26 same-night cells from the comparison in Table~\ref{tab:lora} and its companion adapters (4B and 9B, 48 or 72 tasks, temperature 0.7, four rollouts each): median 0.023, worst 0.042. The range across the four rollouts of a single cell has median 0.055 and reaches 0.094 (a 4B seed-harness cell read 0.171, 0.096, 0.176 and 0.190). Table~\ref{tab:noise} gives the smallest difference between two cells that we treat as resolvable for a given number of rollouts per cell; the paper uses four rollouts and 0.035 everywhere unless stated. The house rule inside the evolution loop is stricter in one direction: a candidate whose 4-rollout mean falls more than 0.02 below the anchor is rejected, which trades some false rejections for protection against ratchet slippage (Appendix~\ref{app:pitfalls}). On WebArena-Lite the per-rollout standard deviation is about 0.045, so the resolvable difference at four rollouts is 0.064. Two configurations of the same model evaluated on different nights, even with the serving stack pinned, differed by up to 0.05, which is why comparisons in this paper are same-night unless marked.

\section{Full 13-Round Evolution Log (Qwen3.5-4B)}
\label{app:log4b}
\input{tab_evolution_rounds_4b}
Table~\ref{tab:evolution-rounds-4b} lists every candidate with its own reading and the fresh anchor measured after the round. Each round proposed two candidates; 18 of 26 passed the judge, one per round was retained (10 in total) and 8 were rejected, with rounds 4, 5 and 12 rejecting both candidates. The three tool edits (rounds 6, 8, 13) and the lessons file (round 2) are the components that the ablation in Appendix~\ref{app:ablation} finds load-bearing. Round 8, a plan-completeness checker tool, is the largest single step (+0.039 on the fresh anchor). One accepted edit turned out to be inert on inspection (the round-3 policy change activated a self-check whose prompt file did not exist) and another silently removed the seed harness's twenty-call brake on identical calls; the judge cannot distinguish an inert edit from a neutral one, which is one reason we report fresh anchors rather than candidate readings. After round 11 the serving context window was raised from 131k to 262k tokens and the current harness was re-anchored (0.305 to 0.269); the final gain is computed from the re-anchored value.

\section{Full 5-Round Evolution Log (Qwen3.5-9B)}
\label{app:log9b}
\input{tab_evolution_rounds_9b}
Table~\ref{tab:evolution-rounds-9b} gives the same record for the 9B line. The 9B line accepted three candidates in five rounds, all of them edits to the distilled lessons file (plan-emission format, hotel-location checks, standalone plan emission); every tool the proposer wrote was rejected because the agent never called it, and the last four candidates read 0.018--0.033 below the anchor and were rejected. The closing harness is therefore the seed harness plus one lessons file, which is why the 9B ablation (Appendix~\ref{app:ablation}) can isolate the whole gain of the line in one cell. An earlier 9B line that we do not report locked its anchor on a single lucky evaluation and then rejected eight consecutive candidates; the fresh-anchor protocol was introduced after it.

\section{Evolution Summary Tables and Per-Round Failure Composition}
\label{app:frates}
\input{tab_evolution_summary}
\input{tab_f_rates}
\input{tab_decomp}
Table~\ref{tab:evolution-summary} re-measures the seed and closing harness of both lines on both splits; Tables~\ref{tab:f-rates} and~\ref{tab:decomp} give the seed and closing compositions and the delivery / quality decomposition referred to in Sections~\ref{sec:taxonomy} and~\ref{sec:evolution}. Figure~\ref{fig:frate_breakdown} plots, for each round in which an edit was accepted, the failure composition of the fresh anchor measured after that round.

\section{Training Details, Rank Ablation and Full-Parameter Fine-Tuning}
\label{app:training}
\begin{figure}[t]
\begin{center}
\includegraphics[width=0.98\linewidth]{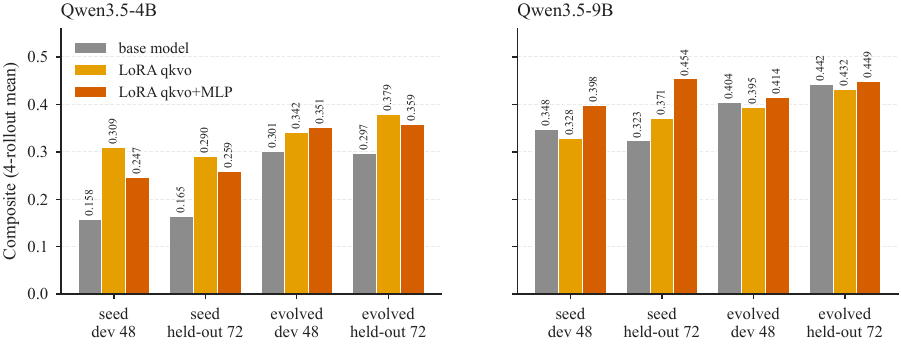}
\end{center}
\caption{\textbf{Base model and LoRA adapters under the seed and the evolved harness}, same machine, 4-rollout means (Table~\ref{tab:lora}; provenance of the evolved-harness held-out base cells in its caption). On the seed harness the adapter alone recovers as much as harness evolution; on the evolved harness the adapter adds to the 4B score, and for 9B the adapter alone already reaches the evolved-harness level.}
\label{fig:lora}
\end{figure}

\input{tab_training}
\input{tab_f_rates_lora}
\input{tab_placebo}
\begin{figure}[h]
\begin{center}
\includegraphics[width=0.98\linewidth]{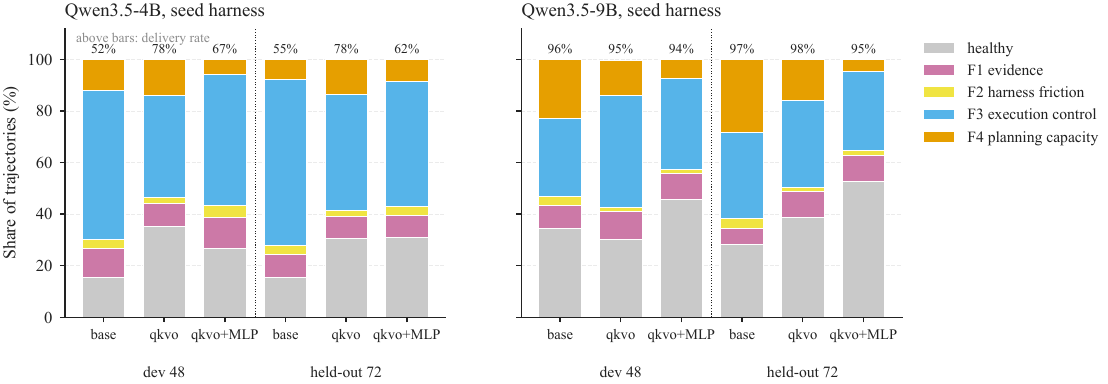}
\end{center}
\caption{\textbf{Failure composition of the base model and the LoRA adapters under the seed harness} (the cells of Table~\ref{tab:f-rates-lora}; numbers above the bars are delivery rates). The 4B adapters raise delivery and shrink the loop class F3; the 9B adapters shrink the content class F4.}
\label{fig:lora-frate}
\end{figure}
All adapters use rank 32, $\alpha=64$, dropout 0.05, AdamW with weight decay 0.01, gradient clipping at 1.0, a constant learning rate of $10^{-4}$ and sixteen step samples per optimiser step; the adapters in Table~\ref{tab:lora} were trained for two epochs and taken at the end of the first, because the second epoch memorises. The placebo permutes the assistant outputs across samples while leaving every input unchanged, so the adapter sees the same prompts, the same harness structure and the same number of steps. The rank ablation and the full-parameter comparison (Table~\ref{tab:training}) use a larger second textbook (198 trajectories over 42 tasks for 4B, 200 over 44 for 9B, one epoch) collected the same way; full-parameter training uses learning rate $2\times10^{-6}$ with 5\% warm-up and cosine decay, fully sharded, and a higher learning rate without warm-up diverged for both sizes. Table~\ref{tab:f-rates-lora} and Figure~\ref{fig:lora-frate} give the failure composition behind finding (iv) of Section~\ref{sec:lora}. On the evolved harness the 9B adapters cut F4 as well (qkvo+MLP: 22\% to 6\% on dev, 19\% to 6\% on held-out), but the freed share reappears as F3 on dev and as F1 on held-out rather than as healthy trajectories, and the mean score is unchanged: the adapter and the evolved lessons file move the same failure class, which is the substitution reported in Section~\ref{sec:lora}.

\section{Component Ablations}
\label{app:ablation}
\input{tab_ablation_9b}
\input{tab_ablation_4b}
Both ablations remove components from the closing harness of the corresponding line. The 9B cells were run on two identical machines on one night (full, minus ledger, minus gate, minus checkpoint on one; minus lessons, minus hard blocks and bare on the other), with a cross-machine offset of at most 0.03 checked by re-running the full configuration on both. The 4B cells were run on one machine in one boot with two rollouts per cell (read as resolvable above 0.05), except the exit-gate cell, which was re-run with four. Removing the ledger in the 9B harness also removes the gate and checkpoint that depend on it. In the 9B minus-ledger cell, 66 of 95 final plans still contained a reconciliation section and 61 cited ledger entries that did not exist. The paired comparison quoted in Section~\ref{sec:evolution} restricts each pair of 4B cells to the tasks both deliver: full versus bare is +0.011 over 35 development tasks and +0.022 over 48 held-out tasks; the seed kit versus bare is $-0.027$ / $-0.009$. The composition of violated scorer checks (14 check types) has the same shape in every 4B cell, and the rate of trajectories that submit several drafts falls from 8.9\% (bare) to 1.0\% (full).

\section{Measurement Pitfalls and Mitigations}
\label{app:pitfalls}
We document three pitfalls encountered in our own measurement process.

\textbf{Winner's curse.} Selecting the best of $K$ candidates inflates the selected reading \citep{winners_curse}. With per-rollout $\sigma \approx 0.023$ and two candidates per round, an accepted candidate's own reading overstates its effect; in the 4B line the candidate readings at rounds 8, 10 and 11 were 0.316, 0.327 and 0.323 while the fresh anchors of the same harnesses were 0.282, 0.288 and 0.305. We report anchors, never candidate readings; Figure~\ref{fig:winners-curse} plots every retained candidate's reading against the fresh anchor of the same harness.

\begin{figure}[h]
\begin{center}
\includegraphics[width=0.36\linewidth]{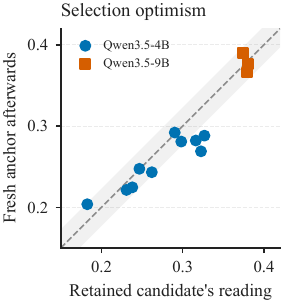}
\end{center}
\caption{\textbf{Selection optimism in the two evolution lines} (Proposition~\ref{prop:optimism}). Each point is a round with an accepted edit: the retained candidate's own 4-rollout reading against the fresh anchor of the same harness measured afterwards; band $= \pm\sigma$ around the diagonal. Points below the diagonal are readings that overstated the harness.}
\label{fig:winners-curse}
\end{figure}

\textbf{Ratchet slippage.} In a sequence of accepts, a lucky round sets the next reference above the true value, and a subsequent neutral edit then reads as a regression and is rejected. Our first 9B line suffered a severe form of this: one outlier accept locked the anchor high and eight consecutive candidates were rejected within noise of it. The mitigation is the fresh anchor: after every accept, the new harness is re-measured before the next candidate is judged.

\textbf{Inert accepts.} An edit that never executes (a middleware wired to the wrong round index, a policy that references a missing file) reads as neutral and can pass the judge, a cousin of the phantom guardrails of \citet{phantom_guardrails}, which fix failures that never happened. One such edit was accepted in the 4B line and three in the WebArena-Lite line. We now require, before an edit is judged, evidence that its code path fired at least once in the evaluation run.

\section{Further Lines in Detail}
\label{app:boundary}
\input{tab_webarena}
Figure~\ref{fig:cross-model} (Section~\ref{sec:other-models}) summarises the nine DeepPlanning lines reported, and Table~\ref{tab:f-rates-other} gives the failure composition of the three smaller non-Qwen ones; the paragraphs below give the mechanism behind each number.

\textbf{Two further models, three lines (August 2026).} Qwen3.6-35B-A3B was run with thinking off from the seed harness and then with thinking on from the thinking-off line's closing harness, and Tencent hy3 with thinking on from the seed harness, one rollout per reading and the judge comparing each candidate with the previous round's reading; the 4-rollout protocol and fresh anchors of Section~\ref{sec:protocol} came later, so the line readings below are single-rollout and the numbers reported in Section~\ref{sec:other-models} are the same-night re-measurements at the end of this paragraph.

Qwen3.6-35B-A3B, thinking off: 25 rounds, 6 accepts, 0.404 to 0.496, with the no-submission rate falling from 23\% to zero and every accept a selective blocker or a tool. Thinking on: 19 rounds, 5 accepts, 0.387 to 0.535, every accept an instruction-layer edit and four of the five on how the ledger is written; on the 72 held-out tasks the closing harness scored 0.512 against 0.325 for the seed (one rollout each). hy3, thinking on: a first batch of 10 rounds in which four legal accepts summed to a net loss (the ratchet slippage of Appendix~\ref{app:pitfalls}), then a clean restart, 5 rounds, 2 accepts, 0.732 to 0.807, again both lessons.

Before finalising the paper we re-measured the seed and closing compositions of all three lines under the full protocol (two runs of two rollouts per cell, temperature 0 as in August). Qwen3.6, thinking off: seed 0.396 (delivery 73\%, healthy 44\%, F1 10\%, F2 32\%, F3 9\%, F4 5\%), closing 0.459 (delivery 91\%, healthy 48\%, F1 14\%, F2 4\%, F3 24\%, F4 10\%), accuracy among delivered plans 0.54 against 0.50: +0.063, all of it delivery, with the harness-friction class F2 removed and the freed trajectories landing mostly in F3. Thinking on: seed 0.266 (delivery 47\%, healthy 26\%, F1 58\%, F3 6\%, F4 10\%), closing 0.320 (delivery 62\%, healthy 29\%, F1 46\%, F3 12\%, F4 13\%), accuracy 0.56 against 0.52: +0.053, again delivery, with the evidence class F1 shrinking. All shares use every trajectory as the denominator. The per-rollout standard deviation in these four cells is 0.02--0.07 against 0.023 for Qwen3.5, so the thinking-off gain is resolvable and the thinking-on gain is within the noise of its own cells. hy3: seed 0.690 (delivery 84\%, healthy 71\%, F1 15\%, F2 6\%, F3 8\%), closing 0.755 (delivery 92\%, healthy 80\%, F1 8\%, F2 5\%, F3 6\%), accuracy 0.818 against 0.824: +0.065, the blank-submission lesson.

In all three lines the gain is delivery, not the accuracy of what is delivered; which process class the harness removes differs by line (F2 for Qwen3.6 off, F1 for Qwen3.6 on and hy3), and in the Qwen3.6 lines the loop class F3 grows as more trajectories run to submission. A 4-rollout re-reading of the Qwen3.6 closing harness four days after the line closed came in at 0.44--0.46 against the single-rollout 0.535; that drift is what led to the same-night rule of Section~\ref{sec:protocol}.

\textbf{Three non-Qwen families.} The same loop, judge and seed harness were run for five rounds (ten candidates) on MiniCPM5-2B, Spark-X2.5-4B and Ministral-3-8B. The diagnosis step of the rule separates these three from the Qwen lines before any evolution: the healthy share is 0--6\% against 15--34\% for the Qwen seeds (Table~\ref{tab:f-rates-lora}), the plans that are delivered score 0.13--0.17 against 0.27--0.31 (Table~\ref{tab:decomp}), and on Spark and Ministral F4 is the largest class at the seed (38\% and 65\%). All three lines closed within noise of their seed anchor, and the composition says why in each case.

MiniCPM5-2B accepted five of ten candidates and gained delivery (61\% to 71\%), but the plans it was pushed to deliver score lower than the ones it delivered already (mean among delivered 0.173 to 0.156), so the product did not move; for this model the seed harness itself had been the large gain (0.047 in the official loop to 0.106, more than double).

Spark-X2.5-4B accepted six candidates that cut F1 from 21\% to 7\% without moving delivery (97\% to 95\%) or the score: the freed trajectories landed in F4 (38\% to 60\%).

Ministral-3-8B starts with F2 and F3 at 7\% together, so the harness has almost no process failure to act on; five of its ten candidates edited the exit gate, eight of the nine rejections were triggered by the proposer's own rollback conditions, and the one accept was a syntax fix.

On all three the loop's candidates went where Section~\ref{sec:evidence_package} predicts, to the enforcement layer. These lines ran with the evidence package as it was before the failure composition was added to it; the compositions in Table~\ref{tab:f-rates-other} are computed after the fact from the anchor trajectories.

\textbf{Gemma~4~12B.} A ten-round line moved the anchor from 0.330 to 0.386, but the anchors were collected across several days on different machines; a same-night re-measurement of seed and evolved harness on the same machine gave +0.019 (dev) and +0.012 (held-out), within noise.

Training on this line's trajectories reproduced the ordering of Table~\ref{tab:lora} at a magnitude within the resolution line: under the seed harness, LoRA gave +0.034 / +0.018 and full-parameter +0.038 / $-0.003$ (dev / held-out), which puts the trained weights under the seed harness where the base model sits under the evolved harness, and stacking them on the evolved harness added nothing (LoRA $-0.042$ / $-0.013$, full-parameter +0.019 / $-0.040$).

This model enters with delivery at 98\% and F4 the largest class (40\% of trajectories), the composition the rule assigns to the weights, so a small harness gain is the branch the rule predicts; the current training recipe, on trajectories from a harness that had barely moved, bought development-task gains that did not carry to held-out, so on this model the weights branch remains open.

For this loop the evidence package was extended to lead with the scorer's aggregated failure composition (the change proposed in Section~\ref{sec:evidence_package}), which is why the Gemma line is the first on which the proposer could see F4-relevant evidence; its gains were nonetheless within noise.

\textbf{WebArena-Lite.} Section~\ref{sec:webarena} and Table~\ref{tab:webarena} report this line. Two details belong here. Three middleware edits accepted in rounds 1--9 had never executed because of a round-indexing bug in the loop (Appendix~\ref{app:pitfalls}); with the bug fixed, the closing composition re-anchored from 0.396 to 0.375, within noise, so the line's gain rests on the interface edits alone. Under the closing harness the undelivered test trajectories stop mostly on the 40-call budget (77 of 108): the closing harness reads more of each page and walks further, which is the observational nature of its gain.
\input{tab_f_rates_other}

\section{Plan-to-JSON Conversion Bias}
\label{app:parser}
DeepPlanning scores a plan only after an LLM converts the agent's free-text plan into a structured form. We audited this step because our agents often submit several drafts of a plan in one answer (25.0\% of the 9B closing harness's answers, counting the 188 scored answers; 1.0\% of the 4B closing harness's). On 20 such multi-draft answers, four converters (the default \texttt{qwen-plus}, DeepSeek V4.1 Flash, Kimi K3 and GPT-5.6 Luna) each silently selected a single draft in 20 of 20 cases, hiding the fact that the agent never converged; only GPT-6 Astra, at 64 times the cost, kept all drafts in 11 cases and selected in 3. A manual audit of 19 answers on which Astra and the default converter disagreed found Astra faithful in 18. Unfaithful conversion inflates the composite score by about 0.025 on average (0.09 on the five hardest answers) relative to faithful conversion. Eleven further converters tested on the five hardest answers against 23 manually established checkpoints did no better than 13 of 23, against 7--9 for the default. The root cause is that the official conversion prompt states no rules and its examples demonstrate rewriting; appending a short rule list to the prompt brings GPT-5.6 Terra to 23 of 23 and the default converter to 20 of 23. We reported the issue and the fix upstream. All numbers in this paper use the default converter, whose repeat agreement with itself is 93.5\% at the default temperature (96.8\% at temperature 0), so every cell is converted the same way; because the bias depends on the multi-draft rate, which differs across configurations (Appendix~\ref{app:ablation}), we report that rate and read differences near 0.025 with care.

\clearpage
\section{Two Frontier API Models Under the Evolved Harness}
\label{app:external}
A co-author ran Kimi K3 and GLM~5.3 as the acting agent on DeepPlanning on a separate machine, four rollouts per cell on the 48 development and 72 held-out tasks, without and with an evolved harness from this paper's lines (Table~\ref{tab:external-models}). Kimi K3 gains +0.058 (development) and +0.049 (held-out); Equation~\ref{eq:identity} on the development means puts 0.010 of it on delivery (96.9\% to 99.3\%) and 0.048 on the quality among delivered plans (0.431 to 0.479), the shape of the 9B line, and mean steps fall by more than half (19.7 to 8.4 and 23.3 to 10.1), the signature of removed loops. GLM~5.3 reaches 0.512 / 0.488 under the evolved harness; without a bare control its harness gain is not identified, and its 0.036 / 0.037 margin over Kimi is a difference between models, not a gain. Per-run readings, run identifiers, serving settings and the mapping to a versioned harness checkpoint are not in the repository at the time of writing and will be released with the code; these cells are therefore not same-night matched with any other cell, carry no noise estimate of their own, and their failure labels come from an earlier labelling scheme and are not pooled with Section~\ref{sec:taxonomy}. They are an aggregate check of harness transfer, not evolution lines, and are not counted among the nine lines of Section~\ref{sec:other-models}.

\begin{table}[H]
\caption{\textbf{Kimi K3 and GLM~5.3 as the acting agent} (four-rollout means on 48 development (D) / 72 held-out (H) tasks, run by a co-author on a separate machine; quality = mean score among delivered development plans). GLM~5.3 has no bare control; no confidence intervals are given because per-run readings are not attached.}
\label{tab:external-models}
\begin{center}
\begin{small}
\begin{tabular}{llccccc}
\toprule
Model & Harness & D score & H score & Delivery D / H & Mean steps D / H & Quality D\\
\midrule
Kimi K3 & Bare & 0.418 & 0.402 & 96.9\% / 95.2\% & 19.7 / 23.3 & 0.431\\
Kimi K3 & Evolved & \textbf{0.476} & \textbf{0.451} & 99.3\% / 98.7\% & 8.4 / 10.1 & 0.479\\
GLM 5.3 & Evolved & 0.512 & 0.488 & 99.6\% / 99.1\% & 7.6 / 9.2 & 0.514\\
\bottomrule
\end{tabular}
\end{small}
\end{center}
\end{table}

\end{document}

%% file: math_commands.tex
\usepackage{amsmath,amsfonts,bm}

\def\eqref#1{equation~\ref{#1}}

\def\1{\bm{1}}

\DeclareMathAlphabet{\mathsfit}{\encodingdefault}{\sfdefault}{m}{sl}
\SetMathAlphabet{\mathsfit}{bold}{\encodingdefault}{\sfdefault}{bx}{n}



%% file: tab_lora_results.tex
\begin{table}[t]
\caption{Base model and LoRA adapters trained on evolved-harness trajectories, under the seed and the evolved harness, on the 48 development and 72 held-out tasks. 4 rollouts per cell on one machine and one night, except the base-model cells under the evolved harness on held-out tasks, the lines' closing anchors from two nights earlier (4B: 2 rollouts); differences below 0.035 (0.05 against a 2-rollout cell) are within noise. Base-row delivery rates differ from Table~\ref{tab:decomp} (a different night).}
\label{tab:lora}
\begin{center}
\begin{small}
\resizebox{\linewidth}{!}{\begin{tabular}{llcccccccc}
\toprule
 & & \multicolumn{4}{c}{Qwen3.5-4B} & \multicolumn{4}{c}{Qwen3.5-9B}\\
 & & \multicolumn{2}{c}{Seed harness} & \multicolumn{2}{c}{Evolved harness} & \multicolumn{2}{c}{Seed harness} & \multicolumn{2}{c}{Evolved harness}\\
Quantity & Weights & Dev & Held-out & Dev & Held-out & Dev & Held-out & Dev & Held-out\\ \midrule
Score & Base & 0.158 & 0.164 & 0.301 & 0.297 & 0.348 & 0.323 & 0.404 & 0.442\\
Score & LoRA qkvo & \textbf{0.309} & \textbf{0.290} & 0.342 & \textbf{0.379} & 0.327 & 0.371 & 0.395 & 0.433\\
Score & LoRA qkvo+MLP & 0.247 & 0.259 & \textbf{0.351} & 0.359 & \textbf{0.398} & \textbf{0.454} & 0.414 & 0.449\\
\addlinespace
Delivery rate & Base & 52\% & 55\% & 90\% & 90\% & 96\% & 97\% & 98\% & 99\%\\
Delivery rate & LoRA qkvo & 78\% & 78\% & 95\% & 97\% & 95\% & 98\% & 100\% & 98\%\\
Delivery rate & LoRA qkvo+MLP & 67\% & 62\% & 90\% & 90\% & 94\% & 95\% & 96\% & 98\%\\
\bottomrule
\end{tabular}}
\end{small}
\end{center}
\end{table}

%% file: tab_noise.tex
\begin{table}[!htbp]
\caption{Smallest difference between two cells that we treat as resolvable, as a function of the number of rollouts averaged per cell, given a per-rollout standard deviation of 0.023 (median over 26 same-night DeepPlanning cells; worst cell 0.042). The paper uses four rollouts per cell.}
\label{tab:noise}
\begin{center}
\begin{small}
\begin{tabular}{cc}
\toprule
Rollouts per cell & Resolvable difference\\ \midrule
1 & 0.07\\
2 & 0.05\\
\textbf{4} & \textbf{0.035}\\
8 & 0.025\\
16 & 0.02\\
\bottomrule
\end{tabular}
\end{small}
\end{center}
\end{table}

%% file: tab_evolution_rounds_4b.tex
\begin{table}[!htbp]
\caption{Qwen3.5-4B evolution line, round by round. Each round proposed two candidates; the retained edit is listed with the harness surface it touched, the other candidate with its verdict. Anchor = fresh 4-rollout mean of the retained harness on the 48 development tasks, measured after the round; F3 = share of loop failures in that anchor. Rounds 4, 5 and 12 rejected both candidates and kept the previous anchor.}
\label{tab:evolution-rounds-4b}
\begin{center}
\begin{scriptsize}
\resizebox{\linewidth}{!}{\begin{tabular}{clllcc}
\toprule
Round & Retained edit & Surface & Other candidate (verdict) & Anchor & F3\\ \midrule
0 & seed harness & & & 0.156 & 64.1\%\\
1 & plan\_submission\_format\_fix & prompts & malformed\_plan\_tag\_middleware (accepted, not retained) & 0.204 & 66.7\%\\
2 & plan\_emission\_lessons (new lessons file) & prompts & plan\_tool\_call\_interceptor (accepted, not retained) & 0.221 & 42.2\%\\
3 & gate\_self\_check\_activation & policies & plan\_envelope\_rescue\_middleware (accepted, not retained) & 0.225 & 38.0\%\\
4 & none & & plan\_format\_middleware, plan\_submission\_tool (both rejected) & 0.225 & \\
5 & none & & coordinate\_fallback\_lesson, plan\_truncation\_guard (both rejected) & 0.225 & \\
6 & coordinate\_lookup\_fallback\_tool & tools & plan\_emission\_force\_syntax\_fix (accepted, not retained) & 0.247 & 35.9\%\\
7 & citation\_format\_normalizer & middleware & plan\_truncation\_rescue\_middleware (rejected) & 0.243 & 33.3\%\\
8 & plan\_completeness\_checker\_tool & tools & budget\_constraint\_checker\_tool (accepted, not retained) & 0.282 & 35.4\%\\
9 & checkpoint\_fuse\_activation & policies & plan\_submission\_format\_enforcement (accepted, not retained) & 0.281 & 38.5\%\\
10 & plan\_submission\_checkpoint & middleware & coordinate\_query\_fallback\_middleware (accepted, not retained) & 0.288 & 37.8\%\\
11 & plan\_truncation\_detection & middleware & citation\_format\_normalizer\_enhancement (accepted, not retained) & 0.305 / 0.269$^\dagger$ & 33.3\%\\
12 & none & & threshold\_reduction, fabricated\_info\_gate\_blocker (both rejected) & 0.269 & \\
13 & flight\_time\_sorter\_tool & tools & citation\_bracket\_enforcement (rejected) & \textbf{0.292} & 37.0\%\\
\bottomrule
\end{tabular}}
\end{scriptsize}
\end{center}
\vskip 0.3em
\noindent\scriptsize $^\dagger$ 0.305 with the 131k-token serving context; the same harness re-anchored at 0.269 after the context window was raised to 262k tokens, and all later rounds use the 262k setting.
\end{table}

%% file: tab_evolution_rounds_9b.tex
\begin{table}[!htbp]
\caption{Qwen3.5-9B evolution line, candidate by candidate (5 rounds, 10 candidates, 3 accepted). Reading = the candidate's own 4-rollout mean when judged; anchor = fresh 4-rollout mean of the retained harness after the round. All three accepted edits modify the distilled lessons file; every proposed tool was rejected because the agent never called it.}
\label{tab:evolution-rounds-9b}
\begin{center}
\begin{small}
\resizebox{\linewidth}{!}{\begin{tabular}{cllclc}
\toprule
Round & Candidate & Surface & Reading & Verdict & Anchor after round\\ \midrule
0 & seed harness & & & & 0.305\\
1 & plan\_emission\_format\_lessons & prompts & 0.380 & accepted & \multirow{2}{*}{0.377}\\
1 & query\_constraint\_extractor\_tool & tools & 0.335 & rejected (never called) & \\
2 & hotel\_location\_lessons & prompts & 0.379 & accepted & \multirow{2}{*}{0.366}\\
2 & hotel\_location\_validator\_tool & tools & 0.399 & rejected (never called) & \\
3 & standalone\_plan\_emission\_lessons & prompts & 0.374 & accepted & \multirow{2}{*}{\textbf{0.390}}\\
3 & plan\_submission\_reminder\_tool & tools & 0.378 & rejected (never called) & \\
4 & query\_constraint\_extractor\_v2 & tools & 0.372 & rejected ($-0.018$ vs anchor) & \multirow{2}{*}{0.390}\\
4 & route\_fabrication\_lessons & prompts & 0.357 & rejected ($-0.033$ vs anchor) & \\
5 & fabricated\_data\_verification\_lessons & prompts & 0.361 & rejected ($-0.029$ vs anchor) & \multirow{2}{*}{0.390}\\
5 & plan\_emission\_force\_threshold\_reduction & middleware & 0.359 & rejected ($-0.031$ vs anchor) & \\
\bottomrule
\end{tabular}}
\end{small}
\end{center}
\end{table}

%% file: tab_evolution_summary.tex
\begin{table}[!htbp]
\caption{Harness evolution before and after. Score = 4-run mean composite; delivery rate = fraction of tasks on which a parsable plan was returned. Seed harness and final evolved harness were re-evaluated outside the loop on the same machine, on the 48 development tasks and the 72 held-out tasks; the evolved-harness held-out cells are the lines' closing anchors measured two nights before the other three cells (4B: 2 rollouts). The 4B line ran 13 rounds (10 accepted edits); the 9B line ran 5 rounds (3 accepted edits). The 4-run discrimination threshold is 0.035; smaller differences are within noise.}
\label{tab:evolution-summary}
\begin{center}
\begin{small}
\begin{tabular}{llcccccc}
\toprule
 & & \multicolumn{3}{c}{Development (48 tasks)} & \multicolumn{3}{c}{Held-out (72 tasks)}\\
Model & Quantity & Seed & Evolved & Gain & Seed & Evolved & Gain\\ \midrule
Qwen3.5-4B & Score & 0.158 & 0.301 & \textbf{+0.143} & 0.164 & 0.297 & \textbf{+0.133}\\
Qwen3.5-4B & Delivery rate & 52\% & 90\% & +38 pp & 55\% & 90\% & +35 pp\\[3pt]
Qwen3.5-9B & Score & 0.348 & 0.404 & \textbf{+0.056} & 0.323 & 0.442 & \textbf{+0.119}\\
Qwen3.5-9B & Delivery rate & 96\% & 98\% & +2 pp & 97\% & 99\% & +2 pp\\
\bottomrule
\end{tabular}
\end{small}
\end{center}
\end{table}

%% file: tab_f_rates.tex
\begin{table}[!htbp]
\caption{Failure composition of the seed harness and of the final evolved harness on the 48 development tasks (share of trajectories; fresh anchors, 4 rollouts each except the 9B seed row, which is the 6-rollout line baseline; labels as defined in Section~\ref{sec:taxonomy}). Rows sum to 100\% up to rounding. The largest change in both lines is the fall in loop failures (F3); F4 is not reduced and ends as the largest failure class for 9B and the second-largest, after the residual F3, for 4B.}
\label{tab:f-rates}
\begin{center}
\begin{small}
\resizebox{\linewidth}{!}{\begin{tabular}{llccccc}
\toprule
Model & Harness & Healthy & F1 evidence & F2 harness friction & F3 execution control & F4 planning capacity\\ \midrule
Qwen3.5-4B & Seed & 14.3\% & 8.6\% & 2.9\% & \textbf{64.1\%} & 9.9\%\\
Qwen3.5-4B & Evolved (round 13) & 31.8\% & 3.1\% & 0.0\% & \textbf{37.0\%} & \textbf{28.1\%}\\
\addlinespace
Qwen3.5-9B & Seed & 26.7\% & 5.2\% & 2.1\% & \textbf{38.9\%} & 27.1\%\\
Qwen3.5-9B & Evolved (round 3) & 40.6\% & 17.7\% & 0.0\% & 16.7\% & \textbf{25.0\%}\\
\bottomrule
\end{tabular}}
\end{small}
\end{center}
\end{table}

%% file: tab_decomp.tex
\begin{table}[!htbp]
\caption{Score decomposition on the 48 development tasks: composite score $=$ delivery rate $\times$ mean score among delivered plans. Delivery = the trajectory returned a plan the scorer could parse. Seed and evolved anchors of each evolution line (4 rollouts; the 9B seed row is the 4-rollout subset of the 6-rollout line baseline). MiniCPM5-2B is included as the case in which the two factors cancel.}
\label{tab:decomp}
\begin{center}
\begin{small}
\resizebox{\linewidth}{!}{\begin{tabular}{llccc}
\toprule
Model & Harness & Composite score & Delivery rate & Score among delivered\\ \midrule
Qwen3.5-4B & Seed & 0.156 & 59\% & 0.266\\
Qwen3.5-4B & Evolved (13 rounds) & 0.292 & \textbf{88\%} & 0.332\\
\addlinespace
Qwen3.5-9B & Seed & 0.299 & 97\% & 0.309\\
Qwen3.5-9B & Evolved (5 rounds) & 0.390 & 99\% & \textbf{0.392}\\
\addlinespace
MiniCPM5-2B & Official benchmark loop & 0.047 & 33\% & 0.141\\
MiniCPM5-2B & Seed & 0.106 & 61\% & 0.173\\
MiniCPM5-2B & Evolved (5 rounds) & 0.110 & 71\% & 0.156\\
\bottomrule
\end{tabular}}
\end{small}
\end{center}
\end{table}

%% file: tab_training.tex
\begin{table}[!htbp]
\caption{Second textbook (198 trajectories over 42 tasks for 4B; 200 over 44 for 9B; one epoch), seed harness, same machine and night per block, 4 rollouts. Top: LoRA versus full-parameter fine-tuning for both sizes. Bottom: rank ablation on 9B. Each adapter in the bottom block is scored against the base model measured on its own machine on the same night (the base drifts by up to 0.035 within seven hours on one machine), so its gains are comparable across ranks but are not differences from the 0.306 / 0.370 row above, and the rank-64 absolute scores are not comparable to the other rows.}
\label{tab:training}
\begin{center}
\begin{small}
\resizebox{\linewidth}{!}{\begin{tabular}{llcccc}
\toprule
Model & Weights & Dev score & Held-out score & Dev gain over base & Held-out gain over base\\ \midrule
Qwen3.5-4B & Base & 0.156 & 0.150 & & \\
Qwen3.5-4B & LoRA qkvo & 0.193 & 0.212 & +0.037 & +0.062\\
Qwen3.5-4B & LoRA qkvo+MLP & 0.280 & 0.267 & +0.124 & +0.117\\
Qwen3.5-4B & Full-parameter & \textbf{0.294} & \textbf{0.311} & +0.138 & +0.161\\
\addlinespace
Qwen3.5-9B & Base & 0.306 & 0.370 & & \\
Qwen3.5-9B & LoRA qkvo+MLP & 0.426 & 0.442 & +0.120 & +0.072\\
Qwen3.5-9B & Full-parameter & \textbf{0.438} & \textbf{0.484} & +0.132 & +0.114\\
\midrule
Qwen3.5-9B & LoRA rank 16, $\alpha = 32$ & 0.425 & 0.424 & +0.091 & +0.058\\
Qwen3.5-9B & LoRA rank 32, $\alpha = 64$ & 0.419 & 0.401 & +0.085 & +0.035\\
Qwen3.5-9B & LoRA rank 64, $\alpha = 128$ (second machine) & 0.415 & 0.430 & +0.116 & +0.085\\
Qwen3.5-9B & Full-parameter (same night as ranks) & 0.447 & 0.456 & +0.139 & +0.093\\
\bottomrule
\end{tabular}}
\end{small}
\end{center}
\end{table}

%% file: tab_f_rates_lora.tex
\begin{table}[!htbp]
\caption{Failure composition of the base model and the LoRA adapters under the seed harness, same machine and night as Table~\ref{tab:lora} (share of trajectories, 4 rollouts; labels as in Section~\ref{sec:taxonomy}; rows sum to 100\% up to rounding). The 4B adapter moves delivery and F3; the 9B adapter moves F4.}
\label{tab:f-rates-lora}
\begin{center}
\begin{small}
\resizebox{\linewidth}{!}{\begin{tabular}{lllcccccc}
\toprule
Model & Adapter & Split & Delivery & Healthy & F1 evidence & F2 harness friction & F3 execution control & F4 planning capacity\\ \midrule
Qwen3.5-4B & none & dev 48 & 52\% & 15.6\% & 10.9\% & 3.6\% & \textbf{57.8\%} & 12.0\%\\
Qwen3.5-4B & LoRA qkvo & dev 48 & \textbf{78\%} & 35.4\% & 8.9\% & 2.1\% & \textbf{39.6\%} & 14.1\%\\
Qwen3.5-4B & LoRA qkvo+MLP & dev 48 & 67\% & 26.6\% & 12.0\% & 4.7\% & 51.0\% & 5.7\%\\
Qwen3.5-4B & none & held-out 72 & 55\% & 15.3\% & 9.0\% & 3.5\% & \textbf{64.6\%} & 7.6\%\\
Qwen3.5-4B & LoRA qkvo & held-out 72 & \textbf{78\%} & 30.6\% & 8.7\% & 2.1\% & \textbf{45.1\%} & 13.5\%\\
Qwen3.5-4B & LoRA qkvo+MLP & held-out 72 & 62\% & 30.9\% & 8.7\% & 3.5\% & 48.6\% & 8.3\%\\
\addlinespace
Qwen3.5-9B & none & dev 48 & 96\% & 34.4\% & 8.9\% & 3.6\% & 30.2\% & \textbf{22.9\%}\\
Qwen3.5-9B & LoRA qkvo & dev 48 & 95\% & 30.2\% & 10.9\% & 1.6\% & 43.2\% & 13.5\%\\
Qwen3.5-9B & LoRA qkvo+MLP & dev 48 & 94\% & 45.8\% & 9.9\% & 1.6\% & 35.4\% & \textbf{7.3\%}\\
Qwen3.5-9B & none & held-out 72 & 97\% & 28.1\% & 6.2\% & 3.8\% & 33.7\% & \textbf{28.1\%}\\
Qwen3.5-9B & LoRA qkvo & held-out 72 & 98\% & 38.5\% & 10.4\% & 1.4\% & 33.7\% & 16.0\%\\
Qwen3.5-9B & LoRA qkvo+MLP & held-out 72 & 96\% & 52.8\% & 10.1\% & 1.7\% & 30.9\% & \textbf{4.5\%}\\
\bottomrule
\end{tabular}}
\end{small}
\end{center}
\end{table}

%% file: tab_placebo.tex
\begin{table}[t]
\caption{Placebo control on Qwen3.5-9B: the real adapter and a placebo adapter trained with the identical recipe (rank 32, $\alpha = 64$, qkvo+MLP, one epoch) on the same trajectories, except that the placebo's assistant outputs are permuted across samples. All cells on the same machine and night, seed harness, 4 rollouts. Rounds = mean number of agent turns per trajectory, pooled over both splits; the placebo hit the 100-turn cap in 9 of 96 development and 44 of 288 held-out trajectories.}
\label{tab:placebo}
\begin{center}
\begin{small}
\resizebox{\linewidth}{!}{\begin{tabular}{lccccc}
\toprule
 & \multicolumn{2}{c}{Development (48 tasks)} & \multicolumn{2}{c}{Held-out (72 tasks)} & \\
Weights & Score & Delivery rate & Score & Delivery rate & Rounds per trajectory\\ \midrule
Base (no adapter) & 0.333 & 98\% & 0.379 & 99\% & 19\\
Real LoRA & 0.399 & not recorded & 0.426 & not recorded & not recorded\\
Placebo LoRA & \textbf{0.175} & 86\% & \textbf{0.152} & 76\% & 52--65\\
\addlinespace
Placebo minus base & $-0.158$ & $-12$ pp & $-0.227$ & $-23$ pp & $+33$ to $+46$\\
\bottomrule
\end{tabular}}
\end{small}
\end{center}
\end{table}

%% file: tab_ablation_9b.tex
\begin{table}[!htbp]
\caption{Component ablation on Qwen3.5-9B (closing harness of the 9B line; seven cells, 4 rollouts each, two identical machines on one night, cross-machine offset $\leq 0.03$). ``Removes ledger'' also removes the exit gate and the checkpoint, which depend on it. Differences below 0.035 are within noise.}
\label{tab:ablation-9b}
\begin{center}
\begin{small}
\resizebox{\linewidth}{!}{\begin{tabular}{llcccccc}
\toprule
Cell & Removes & Dev score & Held-out score & Dev vs full & Held-out vs full & Dev delivery & Dev score among delivered\\ \midrule
Full & nothing & 0.366 & 0.438 & & & 99\% & 0.368\\
$-$ Checkpoint & mid-trajectory checkpoint & 0.385 & 0.440 & +0.019 & +0.002 & & \\
$-$ Gate & exit gate & 0.384 & 0.410 & +0.017 & $-0.028$ & & \\
$-$ Hard blocks & four hard blocks & 0.383 & 0.432 & +0.017 & $-0.005$ & & \\
$-$ Lessons & evolved lessons file & 0.317 & 0.362 & $-0.049$ & $-0.076$ & & \\
Bare & everything & 0.290 & 0.346 & $-0.076$ & $-0.092$ & 85\% & 0.340\\
$-$ Ledger & state ledger (with gate and checkpoint) & 0.278 & 0.292 & $-0.088$ & $-0.145$ & 98\% & 0.284\\
\bottomrule
\end{tabular}}
\end{small}
\end{center}
\end{table}

%% file: tab_ablation_4b.tex
\begin{table}[!htbp]
\caption{Component ablation on Qwen3.5-4B (closing harness of the 4B line; ten cells, one machine, one boot, 2 rollouts per cell except the exit-gate cell with 4). With two rollouts, differences below 0.05 are within noise. ``Seed kit'' keeps the seed harness and removes every product of evolution.}
\label{tab:ablation-4b}
\begin{center}
\begin{small}
\resizebox{\linewidth}{!}{\begin{tabular}{llcccc}
\toprule
Cell & Removes & Dev score & Held-out score & Dev vs full & Held-out vs full\\ \midrule
Full & nothing & 0.264 & 0.314 & & \\
$-$ Evolved middleware & citation auto-fix, submission nag & 0.310 & 0.313 & +0.046 & $-0.001$\\
$-$ Evolved rewrites & three policy and prompt rewrites & 0.302 & 0.287 & +0.038 & $-0.027$\\
$-$ Checkpoint & mid-trajectory checkpoint & 0.270 & 0.330 & +0.007 & +0.016\\
$-$ Gate (4 rollouts) & exit gate & 0.290 & 0.237 & +0.026 & $-0.077$\\
$-$ Ledger & state ledger & 0.253 & 0.254 & $-0.011$ & $-0.060$\\
$-$ Lessons & evolved lessons file & 0.197 & 0.194 & $-0.067$ & $-0.120$\\
$-$ Tools & three invented tools & 0.193 & 0.224 & $-0.071$ & $-0.090$\\
Seed kit & all products of evolution & 0.160 & 0.187 & $-0.104$ & $-0.127$\\
Bare & everything & 0.150 & 0.132 & $-0.114$ & $-0.183$\\
\bottomrule
\end{tabular}}
\end{small}
\end{center}
\end{table}

%% file: tab_webarena.tex
\begin{table}[!htbp]
\caption{WebArena-Lite with Qwen3.5-9B: seed harness versus the closing harness of the 13-round line, and the adapters trained on the closing harness's trajectories (4 rollouts per cell; the resolvable difference on this benchmark is 0.064). Delivery = the agent returned an answer; healthy, F3 and F4 = share of trajectories labelled as in Section~\ref{sec:taxonomy}; accuracy = judged-correct share among delivered answers. Base-model cells: one night, one site host; adapter cells: the following night on that host and a clone, with the closing-harness base anchor on the development tasks re-measured alongside them (0.375; 0.396 the night before).}
\label{tab:webarena}
\begin{center}
\begin{small}
\resizebox{\linewidth}{!}{\begin{tabular}{llccccccc}
\toprule
& & Dev score & \multicolumn{6}{c}{Test, 117 unseen tasks}\\
\cmidrule(lr){4-9}
Harness & Weights & (48 tasks) & Score & Delivery & Healthy & F3 & F4 & Accuracy among delivered\\ \midrule
Seed & Base & 0.271 & 0.256 & 49\% & 23.1\% & 45.9\% & 25.6\% & 0.47\\
Evolved & Base & 0.375 & \textbf{0.344} & 77\% & \textbf{33.8\%} & 40.2\% & 18.2\% & 0.44\\
\addlinespace
Seed & LoRA & 0.253 & 0.216 & 42\% & 20.3\% & 52.6\% & 19.4\% & 0.49\\
Seed & Full-parameter & 0.250 & 0.231 & 45\% & 20.7\% & 47.0\% & 23.9\% & 0.46\\
Evolved & LoRA & 0.349 & 0.327 & 81\% & 31.4\% & 38.7\% & 17.1\% & 0.39\\
Evolved & Full-parameter & 0.406 & 0.331 & \textbf{88\%} & 32.9\% & 35.7\% & 21.2\% & 0.37\\
\bottomrule
\end{tabular}}
\end{small}
\end{center}
\end{table}

%% file: tab_f_rates_other.tex
\begin{table}[!htbp]
\caption{Failure composition of the three smaller non-Qwen lines at their seed and closing anchors (48 development tasks, 4 rollouts; share of trajectories, labels as in Section~\ref{sec:taxonomy}; rows sum to 100\% up to rounding; computed after the fact from the anchor trajectories; the Ministral seed anchor logged 178 of 192 trajectories). Compare the Qwen seed rows of Table~\ref{tab:f-rates-lora}: healthy 15--34\%, F4 12--28\%. On all three the healthy share is at most 6\%, and on Spark and Ministral F4 is the largest class before any evolution.}
\label{tab:f-rates-other}
\begin{center}
\begin{small}
\resizebox{\linewidth}{!}{\begin{tabular}{llcccccccc}
\toprule
Model & Harness & Score & Delivery & Score among delivered & Healthy & F1 evidence & F2 harness friction & F3 execution control & F4 planning capacity\\ \midrule
MiniCPM5-2B & seed & 0.106 & 61\% & 0.173 & 5.7\% & \textbf{42.7\%} & 2.1\% & 32.3\% & 17.2\%\\
MiniCPM5-2B & evolved (5 rounds) & 0.110 & \textbf{71\%} & 0.156 & 4.7\% & 35.9\% & 2.6\% & 34.4\% & 22.4\%\\
\addlinespace
Spark-X2.5-4B & seed & 0.153 & 97\% & 0.158 & 6.2\% & 20.8\% & 2.1\% & 33.3\% & \textbf{37.5\%}\\
Spark-X2.5-4B & evolved (5 rounds) & 0.155 & 95\% & 0.164 & 5.2\% & 6.8\% & 0.0\% & 28.1\% & \textbf{59.9\%}\\
\addlinespace
Ministral-3-8B & seed & 0.112 & 90\% & 0.125 & 0.0\% & 28.1\% & 0.0\% & 7.3\% & \textbf{64.6\%}\\
Ministral-3-8B & evolved (5 rounds) & 0.127 & 93\% & 0.136 & 3.1\% & 30.7\% & 0.5\% & 5.2\% & \textbf{60.4\%}\\
\bottomrule
\end{tabular}}
\end{small}
\end{center}
\end{table}